\documentclass{article}

\usepackage{PRIMEarxiv}

\usepackage[utf8]{inputenc} 
\usepackage[T1]{fontenc}    
\usepackage{amsfonts,amsmath,amssymb,amsthm}
\usepackage{physics,mathtools,bm,mathrsfs}
\usepackage{xcolor,enumerate,autobreak}
\usepackage{wrapfig,subfigure,multirow}
\usepackage{geometry,newclude,appendix}
\usepackage{hyperref,prettyref}
\usepackage[numbers,square, comma, sort&compress]{natbib}

\theoremstyle{plain}
\newtheorem{theorem}{Theorem}[section]
\newtheorem{lemma}[theorem]{Lemma}
\newtheorem{corollary}[theorem]{Corollary}

\theoremstyle{definition}
\newtheorem{proposition}[theorem]{Proposition}
\newtheorem{definition}[theorem]{Definition}
\newtheorem{remark}[theorem]{Remark}
\newtheorem{example}{Example}[section]
\newtheorem{assumption}{Assumption}

\newrefformat{eq}{\textup{(\ref{#1})}}
\newrefformat{lem}{Lemma~\textup{\ref{#1}}}
\newrefformat{thm}{Theorem~\textup{\ref{#1}}}
\newrefformat{cor}{Corollary~\textup{\ref{#1}}}
\newrefformat{prop}{Proposition~\textup{\ref{#1}}}
\newrefformat{assu}{Assumption~\textup{\ref{#1}}}
\newrefformat{rem}{Remark~\textup{\ref{#1}}}
\newrefformat{def}{Definition~\textup{\ref{#1}}}
\newrefformat{sec}{Section~\textup{\ref{#1}}}
\newrefformat{subsec}{Section~\textup{\ref{#1}}}
\newrefformat{example}{Example~\textup{\ref{#1}}}

\def\T{{ \mathrm{\scriptscriptstyle T} }}
\newcommand{\ep}{\varepsilon}
\newcommand{\R}{\mathbb{R}}

\newcommand{\K}{\mathbb{K}}
\newcommand{\caH}{\mathcal{H}}
\newcommand{\caX}{\mathcal{X}}
\newcommand{\bbS}{\mathbb{S}}

\providecommand{\cref}{\prettyref}

\DeclareMathOperator*{\argmin}{arg\,min}
\newcommand{\ang}[1]{\left\langle{#1}\right\rangle}
\DeclareMathOperator{\spn}{span}
\DeclareMathOperator{\ran}{Ran}

\title{Kernel interpolation generalizes poorly
}

\author{
  Yicheng Li, Haobo Zhang \\
  Center for Statistical Science,
Department of Industrial Engineering, Tsinghua University, \\
100084, Beijing, China \\ 
  \texttt{\{liyc22, zhang-hb21\}@mails.tsinghua.edu.cn} \\
   \And
  Qian Lin \\
  Center for Statistical Science,
Department of Industrial Engineering, Tsinghua University, \\
100084, Beijing, China \\ 
  \texttt{qianlin@tsinghua.edu.cn} \\
}

\begin{document}
\maketitle

\begin{abstract}
One of the most interesting problems in the recent renaissance of the studies in kernel regression might be whether kernel interpolation can generalize well, since it may help us understand the `benign overfitting phenomenon' reported in the literature on deep networks. In this paper, under mild conditions, we show that for any $\ep>0$, the generalization error of kernel interpolation
    is lower bounded by $\Omega(n^{-\ep})$.
    In other words, the kernel interpolation generalizes poorly for a large class of kernels.
    As a direct corollary, we can show that overfitted wide neural networks defined on the sphere generalize poorly.
\end{abstract}


 \section{Introduction}
  The `benign overfitting phenomenon' that over-parametrized neural networks can interpolate noisy data but also generalize well is widely observed in the literature on neural networks \citep{neyshabur2015_SearchReal,zhang2017_UnderstandingDeep,belkin2019_ReconcilingModern}.
  This observation challenges the traditional `bias-variance trade-off' doctrine in statistical learning theory \citep{vapnik1999_NatureStatistical} where models interpolating noisy data are supposed to generalize poorly.
  Since \citet{jacot2018_NeuralTangent} illustrated that when the width of neural network tends to infinity, the dynamics of the gradient flow when training the neural network can be approximated by the gradient flow of a kernel regression with respect to the time-invariant neural tangent kernel, we have experienced a renaissance of studies of kernel regression.
  For example, since a number of works
  ~\citep{du2018_GradientDescent,allen-zhu2019_ConvergenceTheory}
  have shown that a wide neural network can overfit any data, whether neural kernel interpolation can generalize might be one of the most important questions towards a theoretical explanation of the `benign overfitting phenomenon'.

  Suppose that we have observed the given samples $(x_1,y_1),\dots,(x_n,y_n)$ and perform regression on $\caH$,
  a reproducing kernel Hilbert space (RKHS) associated with some kernel $k$.
  Traditional results
  ~\citep{caponnetto2007_OptimalRates,andreaschristmann2008_SupportVector}
  in regularized kernel regression, namely \textit{kernel ridge regression}, showed that
  \begin{align}
    \label{eq:KRR}
    \hat{f}_{\lambda} &= \argmin_{f \in \caH} \left[
      \frac{1}{n}\sum_{i=1}^n \{y_i - f(x_i)\}^2 + \lambda \norm{f}_{\caH}^2
      \right],
  \end{align}
  where   $\norm{\cdot}_{\caH}$ is the corresponding norm, is minimax optimal if $\lambda$ is chosen properly according to the sample size;
  subsequent works
  ~\citep{steinwart2009_OptimalRates,fischer2020_SobolevNorm}
  have also extended the optimality results to more general settings.
  In the limiting case where $\lambda\to 0$, we obtain the kernel minimum-norm interpolation~\citep[Section 12.5]{wainwright2019_HighdimensionalStatistics}
  \begin{align}
    \label{eq:KMinNorm}
    \hat{f}_{\mathrm{inter}} = \hat{f}_{0} = \argmin_{f \in \caH} \norm{f}_{\caH} \qq{subject to} f(x_i) = y_i\quad (i = 1,\dots,n).
  \end{align}
  However, there has been little discussion about the generalization performance of $\hat{f}_{\mathrm{inter}}$.

    {
    Recently, many papers have considered the generalization ability of kernel regression in the high dimensional regime
    where the dimension $d$ of the inputs grows with the sample size $n$.
    In this regime, the so-called `benign overfitting phenomenon' has been reported and generalization error bounds are provided
    under certain additional assumptions.
    \citet{liang2020_JustInterpolate} proved that kernel interpolation can generalize when $n \asymp d$ and the data have a low-dimensional structure;
    \citet{liang2020_MultipleDescent} showed similar results when $n = d^{l}$ and $l > 1$ is not an integer.
    Considering a square-integrable regression function,
    \citet{ghorbani2020_LinearizedTwolayers} and its follow-up work \citet{mei2022_GeneralizationError}
    showed that kernel methods can fit at most a degree-$l$ polynomial if $n = d^l$ for $l > 1$;
    \citet{ghosh2021_ThreeStages} proved similar results for the gradient flow with early-stopping.
  }

    {
    On the other hand, under the traditional fixed dimension setting,
    the inconsistency of interpolation has been widely observed and believed~\citep{gyorfi2002_DistributionfreeTheory}.
  }
  Recently, several works claimed the inconsistency of kernel interpolation.
  \citet{rakhlin2018_ConsistencyInterpolation} showed that the kernel interpolation $\hat{f}_{\mathrm{inter}}$ with Laplace kernel is not consistent in odd dimensions;
  \citet{buchholz2022_KernelInterpolation} then extended the result to kernels associated with the Sobolev space $H^s$ for $d/2 < s < 3d/4$
  and all (fixed) dimensions;
  assuming that data are located in uniformly spaced grids,
  \citet{beaglehole2022_KernelRidgeless} established inconsistency results for a class of shift-invariant periodic kernels with mild spectral assumptions.
  Since the neural tangent kernels are not included in all the aforementioned situations,
  it is unclear if these works can help us understand the `benign overfitting phenomenon' in neural network.

  In this paper, we show that under mild assumptions,
  the kernel interpolation $\hat{f}_{\mathrm{inter}}$ generalizes poorly in fixed dimensions.
  To be precise, we prove that
  for any $\ep > 0$, the generalization error
  \begin{align*}
    \norm{\hat{f}_{\mathrm{inter}} - f^*}^2_{L^2} = \Omega(n^{-\ep})
  \end{align*}
  holds with high probability, where $f^*$ is the underlying regression function.
  In subsection \ref{subsec:examples of RKHS}, we further illustrate that our class of kernels not only include the aforementioned kernels but also include the neural tangent kernels.
  In particular, we rigorously show that overfitted wide neural networks on spheres generalize poorly.

  Throughout this paper, we use $L^p(\mathcal{X},\dd \mu)$ to represent the Lebesgue $L^p$ spaces,
  where the corresponding norm is denoted by $\norm{\cdot}_{L^p}$.
  For a function $f$ on $\caX$, we also define the sup-norm $\norm{f}_{\infty} = \sup_{x \in \caX} \abs{f(x)}$,
  which should be distinguished from the $L^\infty$-norm.
  We use the asymptotic notations $O(\cdot)$, $o(\cdot)$, and $\Omega(\cdot)$.
  We also denote $a_n \asymp b_n$ iff $a_n = O(b_n)$ and $a_n = \Omega(b_n)$.

  \section{Preliminaries}

  \subsection{Reproducing kernel Hilbert space}

  Let a compact set $\caX \subset \R^{d}$ be the input space.
  We fix a continuous positive definite kernel $k$ over $\caX$
  and denote by $\caH$ the separable reproducing kernel Hilbert space associated with $k$.
  Let $\rho$ be a probability measure on $\caX \times \R$ and $\mu$ be the marginal probability measure of $\rho$ on $\caX$.
  We denote by $L^2 = L^2(\caX,\dd \mu)$ the Lebesgue space for short.
  Then, it is known~\citep{andreaschristmann2008_SupportVector,steinwart2012_MercerTheorem} that
  the natural embedding $S_\mu : \caH \to L^{2}$ is a Hilbert-Schmidt operator with the Hilbert-Schmidt
  norm $\norm{S_\mu}_{\mathrm{HS}}^2 \leq \kappa^2 = \sup_{x \in \caX} k(x,x)$.
  Moreover, denoting by $S_\mu^* : L^2 \to \caH$ the adjoint operator of $S_\mu$,
  the operator $T = S_\mu S_\mu^* : L^2 \to L^2$
  is an integral operator given by
  \begin{align}
  (Tf)(x)
    = \int_{\mathcal{X}} k(x,y) f(y) \dd \mu(y).
  \end{align}
  Then, it is well known~\citep{caponnetto2007_OptimalRates,steinwart2012_MercerTheorem} that $T$ is self-adjoint, positive and trace-class (thus compact)
  with trace norm $\norm{T}_1 \leq \kappa^2$.
  Thus, the spectral theorem of compact self-adjoint operators and Mercer's theorem
  ~\citep{steinwart2012_MercerTheorem}
  yield that
  \begin{align}
    \label{eq:Mercer}
    T = \sum_{i\in N} \lambda_i \ang{\cdot,e_i}_{L^2} e_i,\quad\quad
    k(x,y) = \sum_{i\in N} \lambda_i e_i(x) e_i(y),
  \end{align}
  where $N \subseteq \mathbb{N}$ is an index set,
  $\left\{ \lambda_i \right\}_{i \in N}$ is the set of positive eigenvalues of $T$ in descending order,
  $e_i$ is the corresponding eigenfunction which is chosen to be continuous, and the convergence of the kernel expansion is absolute and uniform.
  In addition, $\left\{ e_i \right\}_{i \in N}$ forms an orthonormal basis
  of $\overline{\ran S_\mu} \subseteq L^2$
  and $\left\{ \lambda_i^{1/2} e_i \right\}_{i\in N}$ forms an orthonormal basis of
  $\overline{\ran S_\mu^*} \subseteq \caH$.

    {
    The eigenvalues $\lambda_i$ characterize the span of the reproducing kernel Hilbert space and the interplay between $\caH$ and $\mu$.
    We will assume $N = \mathbb{N}$ and impose the following polynomial eigenvalue decay condition:

    \begin{assumption}[Eigenvalue decay]
      \label{assu:EDR}
      There are some $\beta > 1$ and constants $c_1,c_2 > 0$ such that
      \begin{align}
        \label{eq:EDR}
        c_1 i^{- \beta} \leq \lambda_i \leq c_2 i^{-\beta} \quad (i = 1,2,\dots),
      \end{align}
      where $\lambda_i$ is the eigenvalue of $T$ defined in \cref{eq:Mercer}.
    \end{assumption}

    This assumption is quite standard in the literature
    ~\citep{caponnetto2007_OptimalRates,steinwart2009_OptimalRates,blanchard2018_OptimalRates,fischer2020_SobolevNorm,li2023_SaturationEffect}.
    We also remark that the upper bound in \cref{eq:EDR} enables us to provide upper bounds for the convergence rate,
    while the lower bound is essential for the minimax lower bound~\citep{caponnetto2007_OptimalRates} and also for our lower bound.
    This assumption is satisfied by commonly considered kernels including kernels associated with Sobolev RKHS (see \cref{example:Sobolev_RKHS})
    and neural tangent kernels on spheres~\citep{bietti2019_InductiveBias}.

  }

  \subsection{The embedding index $\alpha_{0}$ of an RKHS}\label{subsec:examples of RKHS}

  The decomposition \cref{eq:Mercer} also allows us to introduce the following interpolation spaces
  ~\citep{steinwart2012_MercerTheorem,fischer2020_SobolevNorm,li2023_SaturationEffect}.
  For $s \geq 0$, we define the fractional power $T^s : L^2 \to L^2$
  by
  \begin{align}
    T^s(f) &= \sum_{i\in N} \lambda_i^s \ang{f,e_i}_{L^2} e_i
  \end{align}
  and define the interpolation space $[\caH]^s$ by
  \begin{align}
  [\caH]^s =
    \ran T^{s/2} = \left\{ \sum_{i \in N} a_i \lambda_i^{s/2} e_i ~\Big|~ \sum_{i\in N} a_i^2 < \infty \right\}
    \subseteq L^2,
  \end{align}
  which is equipped with the norm 
  \begin{align}
    \label{eq:GammaNorm}
    \norm{\sum_{i \in N} a_i \lambda_i^{s/2} e_i}_{[\caH]^s} = \left(  \sum_{i\in N} a_i^2  \right)^{1/2}.
  \end{align}
  One may easily verify that $[\caH]^s$ is also a separable Hilbert space
  and $\left\{ \lambda_i^{s/2} e_i \right\}_{i \in N}$ forms an orthonormal basis of it.
  It is clear that we have $[\caH]^0 = \overline{\ran S_{\mu}} \subseteq L^2$ and
  $[\caH]^1 = \overline{\ran S_\mu^*} \subseteq \caH$.
  There are also compact inclusions $[\caH]^{s_1} \hookrightarrow [\caH]^{s_2}$ for $s_1 > s_2 \geq 0$.

  Furthermore, we say $\caH$ has an embedding property of order $\alpha \in (0,1]$
  if $[\mathcal{H}]^\alpha$ can be continuously embedded into $L^\infty(\caX,\dd \mu)$, that is,
  the operator norm
  \begin{equation}
    \label{eq:EMB}
    \norm{[\mathcal{H}]^\alpha \hookrightarrow L^{\infty}(\mathcal{X},\mu)} = M_\alpha < \infty.
  \end{equation}
  \citet[Theorem 9]{fischer2020_SobolevNorm} shows that
  \begin{align}
    \label{eq:EMB_And_InfNorm}
    \norm{[\mathcal{H}]^\alpha \hookrightarrow L^{\infty}(\mathcal{X},\mu)} = \norm{k^\alpha_{\mu}}_{L^\infty},
  \end{align}
  where $\norm{k^\alpha_{\mu}}_{L^\infty}$ is the $L^\infty$-norm of the $\alpha$-power of $k$ given by
  \begin{align}
    \label{eq:Kalpha_Inf}
    \norm{k^\alpha_{\mu}}_{L^\infty}^2 = \operatorname*{ess~sup}_{x \in \caX,~\mu} \sum_{i \in N} \lambda_i^{\alpha} e_i(x)^2.
  \end{align}

  Since we have assumed that $\sup_{x \in \caX}k(x,x) \leq \kappa^2$, we know that \cref{eq:EMB} holds for $\alpha = 1$, i.e., any $\caH$ associated with a bounded kernel function has the embedding property of order $1$.
  By the inclusion relation of interpolation spaces, it is clear that if $\caH$ has the embedding property of order $\alpha$, then it has the embedding 
  properties of order $\alpha'$ for any $\alpha'\geq\alpha$. Thus, we may introduce the following definition
  \begin{definition}
    The embedding index $\alpha_{0}$ of an RKHS $\caH$ is defined by
    \begin{align}
      \label{eq:EMB_Idx}
      \alpha_{0} = \inf\left\{ \alpha :  \norm{[\mathcal{H}]^\alpha \hookrightarrow L^{\infty}(\mathcal{X},\mu)} < \infty  \right\}.
    \end{align}
  \end{definition}

  It is shown in \citet{steinwart2009_OptimalRates} and \citet[Lemma 10]{fischer2020_SobolevNorm} that
  $\alpha_0 \geq \beta$ and the equality holds if the eigenfunctions are uniformly bounded ( i.e.
  $\sup_{i \in N} \norm{e_i}_{L^\infty} < \infty$).
  Throughout this paper, we make the following assumption.

  \begin{assumption}[Embedding index]
    \label{assu:EMB}
    The embedding index $\alpha_0 = 1/\beta$,
    where $\beta$ is the eigenvalue decay in \cref{eq:EDR}.
  \end{assumption}

  \begin{remark}
    \label{rem:EMB}
    Many RKHSs satisfy the embedding index condition. We list several examples below:

    \begin{example}[Sobolev spaces]
      \label{example:Sobolev_RKHS}
      Let $\caX \subset \R^d$ be a bounded domain with smooth boundary
      and suppose $\mu$ is induced by the Lebesgue measure.
      The Sobolev space $H^s(\caX)$ (for a definition, see e.g., \citet{adams2003_SobolevSpaces}) is a reproducing kernel Hilbert space if $s > d/2$.
      Moreover, it can be shown~\citep{fischer2020_SobolevNorm} that
      $H^s(\caX)$ satisfies \cref{assu:EDR} with $\beta= 2s/d$,
      and the interpolation space $[H^s(\caX)]^\alpha \cong H^{\alpha s}(\caX)$.
      By the Sobolev embedding theorem~\citep{adams2003_SobolevSpaces},
      if $\alpha s > d/2$, i.e., $\alpha > d/(2s) = \beta^{-1}$,
      we have
      \begin{align}
        \label{eq:SobolevEmb}
        H^{s\alpha} \hookrightarrow C^{0,\theta}(\caX) \hookrightarrow L^\infty(\caX,\dd \mu),
      \end{align}
      where $C^{0,\theta}(\caX)$ is the Hölder space and $0 < \theta < 2\alpha s / d$.
      It shows that \cref{assu:EMB} is satisfied.
    \end{example}
    \begin{example}[Translation invariant periodic kernel]
      Let $\caX = [-\pi,\pi)^d$ and $\mu$ be the uniform distribution on $\caX$.
      Suppose that the kernel satisfies $k(x,y) = f(x-y \bmod [-\pi,\pi)^d)$, where we denote $a \bmod [-\pi,\pi) = \left\{(a+\pi)\bmod 2\pi \right\} - \pi \in [-\pi,\pi)$ and $x-y \bmod [-\pi,\pi)^d$ is understood to be element-wise.
      Then, it is shown in \citet{beaglehole2022_KernelRidgeless} that
      the Fourier basis $\phi_m(x) = e^{i \ang{m,x}},~m \in \mathbb{Z}^d$ are eigenfunctions of $T$.
      Since $\{\phi_m\}$ are uniformly bounded, \cref{assu:EMB} is satisfied.
    \end{example}

    \begin{example}[Dot-product kernel on spheres]
      \label{example:DotProductKernel}
      Let $\caX = \bbS^{d-1} \subset \R^d$ be the unit sphere
      and assume that $\mu$ is the uniform distribution on $\bbS^{d-1}$.
      Let $k$ be a dot-product kernel, that is, $k(x,y) = f(\ang{x,y})$ for some function $f$.
      Then, it is well known~\citep{smola2000_RegularizationDotproduct} that $k$ can be decomposed using the spherical harmonics $Y_{n,l}$ by
      \begin{align}
        \label{eq:DotProductMercer}
        k(x,y) = \sum_{n=0}^{\infty} \mu_n \sum_{l=1}^{a_n} Y_{n,l}(x)Y_{n,l}(y),
      \end{align}
      where $a_n$ is the dimension of order-$n$ homogeneous harmonic polynomials and $\mu_n$ is an eigenvalue of $T$
      with multiplicity $a_n$.
      Assuming the polynomial eigenvalue decay $\mu_n \asymp n^{-(d-1)\beta}$,
      it is shown in \citet{zhang2023_OptimalityMisspecified} that the embedding index $\alpha_0 = 1/\beta$
      and thus \cref{assu:EMB} holds.
    \end{example}

  \end{remark}

  \subsection{Kernel ridge regression}
  Suppose that we have observed $n$ i.i.d. samples $(x_1,y_1),\dots,(x_n,y_n)$ from $\rho$.
  We denote by $X = (x_1,\dots,x_n)$ the sample inputs and
  $Y = (y_1,\dots,y_n)^\T$.
  Thanks to the representer theorem (see e.g., \citet{andreaschristmann2008_SupportVector}), we could solve the optimization problem \cref{eq:KRR} explicitly:
  \begin{align}
    \label{eq:KRR_Matrix}
    \hat{f}_{\lambda}(x)&= \K(x,X)\left\{ \K(X,X)+n \lambda \right\}^{-1}Y = \frac{1}{n}\K(x,X)(K+\lambda)^{-1}Y
  \end{align}
  where $\K(x,X) = \left( k(x,x_1),\dots,k(x,x_n) \right)$, $\K(X,X) = \big( k(x_i,x_j) \big)_{n\times n}$
  and $K = \K(X,X)/n$ is the normalized kernel matrix.
  Similarly, we also have an explicit formula for the kernel interpolation~\citep[Section 12.5]{wainwright2019_HighdimensionalStatistics}:
  \begin{align}
    \label{eq:KMinNorm_Matrix}
    \hat{f}_{\mathrm{inter}}(x)&=\K(x,X)\K(X,X)^{-1}Y = \frac{1}{n}\K(x,X)K^{-1}Y.
  \end{align}
  Here we notice that \cref{eq:KMinNorm_Matrix} is well-defined since we have assumed that $k$ is positive definite and thus $K$ is invertible.
  For conditions and examples of positive definite kernels, we refer to e.g., \citet{wendland2004_ScatteredData}.

  In the last two decades, the theories of integral operators and empirical processes are utilized to understand the generalization ability of $\hat f_{\lambda}(x)$ from various aspects (e.g., \citet{caponnetto2007_OptimalRates,fischer2020_SobolevNorm,li2023_SaturationEffect}).
  One of the essential steps in these works is an integral operator interpretation of the formula \cref{eq:KRR_Matrix}.
  More precisely, they introduced the sampling operator $K_{x} : \R \to \caH$ given by $K_x y = y k(x,\cdot)$, their adjoint operator
  $K_x^*: \caH \to \R$ given by $K_x^* f = f(x)$
  and the sample covariance operator $T_X : \caH \to \caH$ given by
  \begin{align}
    \label{eq:TX}
    T_X &= \frac{1}{n}\sum_{i=1}^n K_{x_i} K_{x_i}^*.
  \end{align}
  With these operators, there is an alternative expression of $\hat f_{\lambda}(x)$:
  \begin{align}
    \label{eq:KRR_integral}
    \hat{f}_\lambda = (T_X+\lambda)^{-1} g_Z,
  \end{align}
  where $g_Z = n^{-1} (\sum_{i=1}^n K_{x_i}y_i) \in \caH$.



  \section{Theoretical Results}

  \subsection{Main theorem}


  To state the main theorem, we need two more mild conditions on the kernel $k$ and the noise $y-f_{\rho}^{*}(x)$.


  \begin{assumption}
    \label{assu:Holder}
    The kernel $k$ is Hölder-continuous,
    that is, there exists some $s \in (0,1]$ and $L > 0$ such that
    \begin{align}
      \label{eq:Holder}
      \abs{k(x_1,x_2) - k(y_1,y_2)} \leq L \norm{(x_1,x_2) - (y_1,y_2)}_{\R^{d \times d}}^s, \quad
      \forall x_1,x_2,y_1,y_2 \in \caX.
    \end{align}
  \end{assumption}


  \begin{assumption}[Noise]
    \label{assu:noise}
    The conditional variance of the noise satisfies that
    \begin{align}
      E_{(x,y)\sim \rho} \left[ \left\{ y-f^{*}_{\rho}(x) \right\}^2 \mid x \right] \geq \sigma^2 > 0,
      \quad \mu\text{-a.e. } x \in \mathcal{X}.
    \end{align}
  \end{assumption}

  \cref{assu:Holder} requires that $k$ is Hölder continuous and \cref{assu:noise} simply requires that the noise is non-vanishing.
  It is clear that both these assumptions are quite mild.
  For example, \cref{assu:Holder} holds for the Laplacian kernels, RBF kernels, neural tangent kernels, etc.;
  \cref{assu:noise}  holds for the usual regression models $y=f^{*}(x)+\epsilon$ where $\epsilon$ is an independent non-zero noise.

  \begin{theorem}
    \label{thm:NoGeneralization}
    Suppose that the Assumptions~\ref{assu:EDR},\ref{assu:EMB},\ref{assu:Holder} and~\ref{assu:noise} hold.
    There is an absolute constant $c > 0$ such that
    for any $\ep > 0$ and any $\delta \in (0,1)$,
    when $n$ is sufficiently large (depending on $\ep$ and $\delta$),  one has that
    \begin{align*}
      E \left( \norm{\hat{f}_{\mathrm{inter}} - f^*_\rho}_{L^2}^2 \mid X \right)  \geq c \sigma^2 n^{-\ep}
    \end{align*}
    holds with probability at least $1-\delta$ with respect to the random samples.
  \end{theorem}

  \begin{remark}
    In \cref{thm:NoGeneralization}, since the sufficiently large $n$ depends on $\ep$ and $\delta$,
    the statement can not imply the inconsistency of kernel interpolation.
    Nevertheless, this theorem actually suggests that for a fairly large class of positive definite kernels, one can not expect that the kernel interpolation can generalize well.
  \end{remark}

  \subsection{Discussion}\label{subsec:Discussion}
  The `benign overfitting phenomenon' might be one of the most puzzling observations reported in the literature on neural networks.
  Since \citet{jacot2018_NeuralTangent} introduced the neural tangent kernel, it has become a common strategy to study the wide neural network through analyzing the kernel regression. Thus, the `benign overfitting phenomenon' raised a natural question that
  whether the kernel interpolations, especially the neural tangent kernel interpolation, can generalize well?


  There are few works~\citep{rakhlin2018_ConsistencyInterpolation,buchholz2022_KernelInterpolation,beaglehole2022_KernelRidgeless} showed that the kernel interpolation generalized poorly 
  in various settings.
  For example, \citet{rakhlin2018_ConsistencyInterpolation} showed that when the dimension $d$ is a fixed odd number,
  the Laplace kernel interpolation generalizes poorly for random observation with Rademacher noise;
  \citet{buchholz2022_KernelInterpolation} considered the regression with respect to the kernel associated with the Sobolev space $H^{s}$ and showed that the corresponding kernel interpolation generalizes poorly if $d/2<s<3d/4$; \citet{beaglehole2022_KernelRidgeless} further showed that the kernel interpolation with respect to the periodical translation invariant kernel generalizes poorly for the grid data.
  However, to the best of our knowledge, the poor generalization ability of the NTK interpolation has been only obtained for one dimensional grid data in \citet{lai2023_GeneralizationAbility}.

  \cref{thm:NoGeneralization} states that for any kernel $k$ satisfying that the embedding index $\alpha_{0}$ of the RKHS associated to $k$
  equals to $1/\beta$, the inverse of the eigen-decay rate, the corresponding kernel interpolation generalized poorly.
  Note that the listed examples in \cref{subsec:examples of RKHS} clearly contains all the kernels appeared in \citet{buchholz2022_KernelInterpolation,beaglehole2022_KernelRidgeless,rakhlin2018_ConsistencyInterpolation}. Thus, it is clear that our result is more general.
  Moreover, our result also applies to the neural tangent kernel (NTK) on spheres~\citep{bietti2019_InductiveBias,bietti2020_DeepEquals},
  which is not covered in either of the previous works.
  Consequently, we could assert that the neural tangent kernel interpolation and the overfitted wide neural network generalized poorly,
  a statement contradicts to the widely observed `benign overfitting phenomenon'.
  Thus, our result suggests that we may need to explain the `benign overfitting phenomenon' from other perspective.

  \subsection{Proof sketch}
  It is clear that the generalization error of kernel ridge regression (including the interpolation case of $\lambda = 0$) is lower bounded by
  the variance term $V(\lambda)$, i.e.,
  \begin{align*}
    E \left( \norm{\hat{f}_{\lambda} - f^*_\rho}^2_{L^2} \mid X \right)
    \geq V(\lambda) = \frac{\sigma^2}{n^2} \int_{\caX} \K(x,X)(K+\lambda)^{-2}\K(X,x) \dd \mu(x).
  \end{align*}

  Let us rewrite $V(\lambda)$ in the following operator form:
  \begin{align}
    \label{eq:3_V2}
    V(\lambda) =  \frac{\sigma^2}{n} \int_{\caX} \norm{(T_{X}+\lambda)^{-1}k(x,\cdot)}_{L^2,n}^2 \dd \mu(x),
  \end{align}
  where $\norm{f}_{L^2,n}^2 = n^{-1} \{\sum_{i=1}^n f(x_i)^2\}$.
  We then invoke an important observation appeared in \citet{li2023_SaturationEffect} which claims that
  for $\lambda = \Omega( n^{-1/2} )$,
  \begin{align}
    \label{eq:3_VApprox}
    V(\lambda) \approx \frac{\sigma^2}{n} \int_{\caX} \norm{(T+\lambda)^{-1}k(x,\cdot)}_{L^2,n}^2 \dd \mu(x)
    \approx \frac{\sigma^2}{n} \int_{\caX} \norm{(T+\lambda)^{-1}k(x,\cdot)}_{L^2}^2 \dd \mu(x)
    \asymp \frac{\sigma^2 \lambda^{-1/\beta}}{n}.
  \end{align}

  However, the requirement that $\lambda = \Omega( n^{-1/2} )$ is far from enough to show the nearly constant lower bound.
  As one of our major technical contributions in this paper,
  we sharpen the estimation by the embedding index assumption and prove the approximation
  \cref{eq:3_VApprox} actually holds for $\lambda \asymp n^{-\beta+\epsilon}$ for any $\epsilon>0$.
  Combining with the elementary observation $V(0) \geq V(\lambda)$, we get the desired claim.
  We refer to the supplementary material for a complete proof.

  \section{Application to neural networks}
  Suppose that we have observed $n$ i.i.d.\ samples $(x_1,y_1),\dots,(x_n,y_n)$ from $\rho$.
  For simplicity, we further assume that the marginal distribution $\mu$ of $\rho$ is the uniform distribution on the unit sphere $\mathbb{S}^{d-1}$.
  We use a two-layer neural network of width $m$ to perform the regression on $(x_{i},y_{i})$'s.
  More precisely,
  we consider the following two-layer neural network:
  \begin{align}
    f(x;\theta) = \sqrt{\frac{2}{m}}\sum_{r=1}^{m}a_r \sigma\left(w_r^\T x\right),
  \end{align}
  where $\sigma(x) = \max(x,0)$ is the ReLU activation and $\theta = (w_1,\dots,w_m,a_1,\dots,a_m)$ are the parameters.
  With $\theta$ randomly initialized as $\theta(0)$, we consider the training process given by the gradient flow
  $\dot{\theta}(t) = - \partial L / \partial \theta $,
  where the loss function is
  \begin{align}
    L(\theta) = \frac{1}{2n}\sum_{i=1}^n \left\{ y_i - f(x_i;\theta) \right\}^2.
  \end{align}
  Then, when the network is over-parametrized, namely $m$ is large, the theory of lazy training (see e.g., \citet{lai2023_GeneralizationAbility,lee2019_WideNeural}) shows that the trained network
  $\hat{f}^{\mathrm{NN}}_t(x) = f(x;\theta(t))$ can be approximated by a kernel gradient method with respect the following neural tangent kernel
  \begin{align}
    k_{\mathrm{NT}}(x,y) = \frac{2}{\pi}\left(\pi-\arccos\ang{x,y}\right)\ang{x,y} +\frac{1}{\pi}(1 - \ang{x,y}^2)^{1/2}.
  \end{align}
  Particularly, as $t \to \infty$, the network output can be approximated by kernel interpolation.
  The following result is a corollary of
  \citet[Proposition 3.2]{lai2023_GeneralizationAbility}.
  \begin{proposition}
    \label{prop:NNConverge}
    Suppose that we initialize the neural network symmetrically, i.e., the width $m = 2l$ and independently initialize $a_r(0),~w_r(0) \sim N(0,1)$ and $a_{l+r}(0) = -a_{r}(0)$, $w_{l+r}(0) = w_r(0)$ for $r = 1,\dots,l$.
    Then, for any $\delta \in (0,1)$, when $m$ is sufficiently large, with probability at least $1-\delta$ one has
    \begin{align}
      \limsup_{t \to \infty} \sup_{x \in \mathbb{S}^{d-1}}\abs{\hat{f}^{\mathrm{NN}}_t(x)  - \hat{f}_{\mathrm{inter}}^{\mathrm{NTK}}(x)} = o_m(1),
    \end{align}
    where $\hat{f}_{\mathrm{inter}}^{\mathrm{NTK}}$ is the kernel interpolation with respect to $k_{\mathrm{NT}}$.
  \end{proposition}
  Moreover, since $k_{\mathrm{NT}}$ is a dot-product kernel satisfying a polynomial eigenvalue decay of $\beta = d/(d-1)$
  ~\citep{bietti2019_InductiveBias},
  our assumptions are satisfied as discussed in \cref{example:DotProductKernel}.

  \begin{corollary}
    There is a constant $c > 0$ such that
    for any $\ep > 0$ and $\delta \in (0,1)$,
    when $n$ and $m$ are sufficiently large, one has that
    \begin{align*}
      E \left(\liminf_{t \to \infty} \norm{\hat{f}^{\mathrm{NN}}_t(x) - f^*_\rho}_{L^2}^2 \mid X \right)
      \geq c \sigma^2 n^{-\ep}
    \end{align*}
    holds with probability at least $1-\delta$ with respect to the random samples and initialization.
  \end{corollary}

  To the best of our knowledge, such a lower bound for wide neural network has only been proven in \citet{lai2023_GeneralizationAbility} for one-dimensional grid data.
  Our result claims that overfitted wide neural network generalizes poorly on the sphere $\mathbb S^{d-1}$,
  This result shows that overfitted neural networks on the sphere generalize poorly, which also suggests more efforts are needed to understand the mystery of generalization of over-parametrized neural networks.

  \section*{Acknowledgement}
  Authors Yicheng Li and Haobo Zhang contributed equally to this work.
  This work is supported in part by the National Natural Science Foundation of China (Grant 11971257)
  and
  Beijing Natural Science Foundation (Grant Z190001).

\bibliographystyle{kp}  
\bibliography{references}  

\appendix

  \section{Proofs}
  \label{sec:Proof}

  \subsection{More preliminaries}\label{subsec:More_Prelim}
  In the following, we will denote $h_x = k(x,\cdot)$ for $x \in \caX$.
  For a function $f$, we use $f[X]=  (f(x_1),\dots,f(x_n))^\T$.
  To simplify the notation, we will use $c,C$ to represent constants, which may change in the context;
  we will also abbreviate
  \begin{align*}
    T_{\lambda} = T+\lambda,\quad  T_{X\lambda} = T_X + \lambda.
  \end{align*}

  Following \citet{li2023_SaturationEffect}, we consider the sample subspace
  \begin{align}
    \mathcal{H}_n = \spn \left\{ k(x_1,\cdot),\dots,k(x_n,\cdot) \right\} \subset \caH.
  \end{align}
  Then, it is easy to verify that $\ran T_X = \caH_{n}$ and $K$ is the representation matrix of $T_X$
  under the natural basis $\left\{ k(x_1,\cdot),\dots,k(x_n,\cdot) \right\}$.
  Consequently, for any continuous function $\varphi$ we have
  \begin{align}
    \label{eq:TXMatrixForm}
    \varphi(T_X)\K(X,\cdot) = \varphi(K) \K(X,\cdot).
  \end{align}
  where the left-hand side is understood by applying the operator elementwise.
  Since from the property of reproducing kernel Hilbert space we have $\ang{h_x,f}_{\caH} = f(x)$ for $f \in \caH$,
  taking inner product elementwise between \cref{eq:TXMatrixForm} and $f$, we obtain
  \begin{align}
    \label{eq:TXActionF}
    (\varphi(T_X) f)[X] = \varphi(K)f[X].
  \end{align}

  Moreover, for $f,g \in \caH$, we define empirical semi-inner product
  \begin{align}
    \label{eq:SampleInnerProductL2}
    \ang{f,g}_{L^2,n} & = \frac{1}{n}\sum_{i=1}^n f(x_i)g(x_i) = \frac{1}{n}f[X]^\T g[X],
  \end{align}
  and denote by $\norm{\cdot}_{L^2,n}^2$ the corresponding empirical semi-norm.
  Then, \citet{li2023_SaturationEffect} established the following simple but important connection.

  \begin{proposition}
    For $f,g \in \caH$, we have
    \begin{equation}
      \label{eq:SampleInnerProductsRelation}
      \ang{f,g}_{L^2,n} = \ang{T_X f,g}_{\caH} = \ang{T_X^{1/2} f, T_X^{1/2} g}_{\caH}.
    \end{equation}
  \end{proposition}
  \begin{proof}
    Notice that $T_X f = \frac{1}{n}\sum_{i=1}^n f(x_i) h_{x_i}$, and thus
    \begin{align*}
      \ang{T_X f,g}_{\caH} = \frac{1}{n}\sum_{i=1}^n f(x_i) \ang{h_{x_i}, g}_{\caH}
      = \frac{1}{n}\sum_{i=1}^n f(x_i) g(x_i) = \ang{f,g}_{L^2,n}.
    \end{align*}
    The second inequality comes from the definition of $T_X^{1/2}$.
  \end{proof}


  \subsection{The variance term}

  The first step of the proof is to lower bound the generalization error by
  the variance term $V(\lambda)$, whose explicit expression can be derived.
  This step is elementary and quite standard in the literature, so we omit the proof here.

  \begin{proposition}
    \label{prop:VarianceTerm}
    Under \cref{assu:noise}, for any $\lambda \geq 0$, we have
    \begin{align}
      \label{eq:Pf_V}
      E \left(\norm{\hat{f}_{\lambda} - f^*_\rho}^2_{L^2} \mid X \right)
      \geq V(\lambda) = \frac{\sigma^2}{n^2} \int_{\caX} \K(x,X)(K+\lambda)^{-2}\K(X,x) \dd \mu(x).
    \end{align}
  \end{proposition}
  \begin{proof}
    Let us denote $K_{\lambda} = K+\lambda$.
    Recalling \cref{eq:KRR_Matrix}, we have
    \begin{align*}
      \hat{f}_{\lambda} - f^*_\rho
      &= \frac{1}{n} \K(x,X)K_\lambda^{-1} Y\\
      &= \frac{1}{n} \K(x,X)K_\lambda^{-1}\left( f^*_\rho[X] + \bm{\epsilon} \right) - f^*_\rho(x)\\
      &= \frac{1}{n} \K(x,X)K_\lambda^{-1}\bm{\epsilon} + \K(x,X)K_\lambda^{-1} f^*_\rho[X]- f^*_\rho(x),
    \end{align*}
    where $\bm{\epsilon}$ is the column vector of the noises $\epsilon_i = y_i - f^*_\rho(x_i)$.
    Since $\epsilon_i | X$ are independent with mean zero and variance $\sigma^2_{x_i} \geq \sigma^2$,
    simple calculation shows
    \begin{align*}
      E \left( \norm{\hat{f}_{\lambda} - f^*_\rho}^2_{L^2} \mid X \right)
      & = E \left[ \int_{\caX} \left\{ \hat{f}_{\lambda}(x) - f^*_\rho(x) \right\}^2 \dd \mu(x) \mid X \right] \\
      & =  \int_{\caX}E \left[ \left\{ \hat{f}_{\lambda}(x) - f^*_\rho(x) \right\}^2 \mid X \right] \dd \mu(x)  \\
      & \geq \int_{\caX} \mathrm{var}\left\{ \hat{f}_{\lambda}(x)  \mid X \right\}\dd \mu(x) \\
      & = \frac{1}{n^2} \int_{\caX}  \K(x,X)K_\lambda^{-1}
      E \left( \bm{\epsilon}\bm{\epsilon}^\T \mid X \right) K_\lambda^{-1} \K(X,x)\dd \mu(x) \\
      & \geq \frac{\sigma^2}{n^2} \int_{\caX} \K(x,X)K_\lambda^{-2}\K(X,x)\dd \mu(x).
    \end{align*}
  \end{proof}

  The simple but critical observation based on the matrix form in \cref{eq:Pf_V} is that
  \begin{align*}
  (K+\lambda_1)
    ^{-2} \succeq  (K+\lambda_2)^{-2} \qif \lambda_1 \leq \lambda_2,
  \end{align*}
  where $\succeq$ represents the partial order of positive semi-definite matrices,
  and thus we have the following proposition~\citep{li2023_SaturationEffect}.

  \begin{proposition}
    \label{prop:VarOrder}
    For $\lambda_1 \geq \lambda_2 \geq 0$, we have $V(\lambda_1) \leq V(\lambda_2)$.
    Particularly, for any $\lambda \geq 0$,
    \begin{align*}
      V(0) \geq V(\lambda).
    \end{align*}
  \end{proposition}

  This proposition allows us to consider the variance term of slightly bigger $\lambda$'s,
  where concentration results can be established.
  However, it is still difficult to directly analyze the matrix form.
  The key is to rewrite the matrix form into the operator form using empirical semi-norm introduced in
  \citet{li2023_SaturationEffect}.

  \begin{lemma}
    The variance term in \cref{eq:Pf_V} satisfies
    \begin{align}
      \label{eq:VarAlterForm}
      V(\lambda) = \frac{\sigma^2}{n} \int_{\caX} \norm{(T_{X}+\lambda)^{-1}k(x,\cdot)}_{L^2,n}^2 \dd \mu(x)
    \end{align}
  \end{lemma}
  \begin{proof}
    By definition it is obvious that $h_x[X] = \K(X,x)$.
    From \cref{eq:TXActionF}, we find that
    \begin{align*}
      \left( (T_X+\lambda)^{-1}h_x  \right)[X] = (K+\lambda)^{-1} h_x[X] = (K+\lambda)^{-1}\K(X,x),
    \end{align*}
    so
    \begin{align*}
      \frac{1}{n} \K(x,X)(K+\lambda)^{-2}\K(X,x)
      &= \frac{1}{n}\norm{(K+\lambda)^{-1}\K(X,x)}_{\R^n}^2 \\
      &= \frac{1}{n}\norm{\{ (T_X+\lambda)^{-1}h_x\}[X]}_{\R^n}^2 \\
      &= \norm{(T_X+\lambda)^{-1}h_x}_{L^2,n}^2.
    \end{align*}
    from the definition \cref{eq:SampleInnerProductL2} of empirical semi-norm.
  \end{proof}

  The operator form \cref{eq:VarAlterForm} allows us to apply concentration inequalities and establish the following
  two-step approximation~\citep{li2023_SaturationEffect}.
  The main difference is that we allow $\lambda$ to be much smaller than in \citet{li2023_SaturationEffect}.
  \begin{align}
    \label{eq:4_2Step}
    \norm{(T_X+\lambda)^{-1}h_x}_{L^2,n}^2
    \stackrel{A}{\approx}
    \norm{(T+\lambda)^{-1}h_x}_{L^2,n}^2
    \stackrel{B}{\approx}
    \norm{(T+\lambda)^{-1}h_x}_{L^2}^2.
  \end{align}

  \subsection{Approximation B}
  Let us first consider the approximation B in \cref{eq:4_2Step}.
  We first provide norm controls of basis functions $h_x = k(x,\cdot)$ using the embedding condition.
  \begin{lemma}
    Suppose $\caH$ has embedding index $\alpha_0$.
    Let $p,\gamma \geq 0$, $\alpha > \alpha_0$ such that $0 \leq 2-\gamma-\alpha \leq 2p$, then
    \begin{align}
      \norm{T_{\lambda}^{-p} h_x}_{[\caH]^\gamma}^2 \leq M_\alpha^2 \lambda^{2-2p-\gamma-\alpha},\quad \mu\text{-a.e.}~ x\in \caX.
    \end{align}
  \end{lemma}
  \begin{proof}
    Recalling the definition \cref{eq:GammaNorm} and using Mercer's decomposition \cref{eq:Mercer}, we have
    \begin{align*}
      \norm{T_{\lambda}^{-p} h_x}_{[\caH]^\gamma}^2 &=
      \norm{T^{-\gamma/2} T_{\lambda}^{-p} h_x}_{L^2}^2 \\
      &= \norm{ \sum_{i \in N} \lambda_i^{-\gamma/2} (\lambda_i+\lambda)^{-p} \lambda_i e_i(x) e_i}_{L^2}^2 \\
      &= \sum_{i \in N} \lambda_i^{2-\gamma} (\lambda_i+\lambda)^{-2p} e_i(x)^2 \\
      & = \sum_{i \in N} \left\{ \lambda_i^{2-\gamma-\alpha} (\lambda_i+\lambda)^{-2p}  \right\} \lambda_i^\alpha e_i(x)^2 \\
      & \leq \sum_{i \in N} \lambda^{2-\gamma-\alpha-2p} \lambda_i^\alpha e_i(x)^2 \\
      & \leq M_\alpha^2 \lambda^{2-2p-\gamma-\alpha},
    \end{align*}
    where last but second inequality uses \cref{prop:FilterKRRControl} and the final inequality
    uses that fact that
    \begin{align*}
      \operatorname*{ess~sup}_{x \in \caX,~\mu} \sum_{i \in N} \lambda_i^{\alpha} e_i(x)^2
      = \norm{[\mathcal{H}]^\alpha \hookrightarrow L^{\infty}(\mathcal{X},\mu)}^2 = M_\alpha^2,
    \end{align*}
    which is the consequence of \cref{eq:Kalpha_Inf} and \cref{eq:EMB_And_InfNorm}.
  \end{proof}
  Noticing that the $L^\infty$-norm of $f \in \caH$ can be bounded by
  \begin{align*}
    \norm{f}_{L^\infty}^2
    \leq M_\alpha^2 \norm{f}_{[\caH]^\alpha}^2,
  \end{align*}
  we have the following corollary.
  \begin{corollary}
    Suppose $\caH$ has embedding index $\alpha_0$ and $\alpha > \alpha_0$. Then the following holds for $\mu\text{-a.e.}~ x\in \caX$:
    \begin{align}
      \label{eq:RegK_Inf}
      \norm{T_{\lambda}^{-1}h_x}_{L^\infty}^2 & \leq M_{\alpha}^4 \lambda^{-2\alpha}, \\
      \label{eq:RegK_L2}
      \norm{T_{\lambda}^{-1}h_x}_{L^2}^2 & \leq M_{\alpha}^2 \lambda^{-\alpha}, \\
      \label{eq:RegK_H}
      \norm{T_{\lambda}^{-1/2}h_x}_{\caH}^2 & \leq M_{\alpha}^2 \lambda^{-\alpha}.
    \end{align}
  \end{corollary}

  Now we can provide the following estimation using concentration inequalities with a covering number argument.

  \begin{lemma}
    \label{lem:ApproxB}
    Suppose that the embedding index of $\caH$ is $\alpha_0$ and \cref{assu:Holder} holds.
    Fix $\alpha > \alpha_0$ and suppose $\lambda = \lambda(n) \to 0$ satisfies
    $\lambda = \Omega\left(n^{-p}\right)$ for some $p < \infty$.
    Then, for any $\delta \in (0,1)$,
    for sufficiently large $n$ (which will depend on $\delta, p$),
    with probability at least $1-\delta$,
    for $\mu$-a.e. $x \in \caX$
    \begin{align*}
      \frac{1}{2} \norm{T_{\lambda}^{-1} h_x}_{L^2}^2 - R
      \leq \norm{T_{\lambda}^{-1} h_x}_{L^2,n}^2 \leq \frac{3}{2} \norm{T_{\lambda}^{-1} h_x}_{L^2}^2
      + R,
    \end{align*}
    where
    \begin{align*}
      R = R(\lambda,n) = C M_{\alpha}^4 \frac{\lambda^{-2\alpha}}{n} \ln \frac{n}{\delta}
    \end{align*}
    and constant $C$ is independent of $\delta, n$.
  \end{lemma}
  \begin{proof}
    Let us denote $\mathcal{K}_\lambda = \left\{ T_{\lambda}^{-1} h_x \right\}_{x \in \caX}$.
    By \cref{lem:CoveringRegularK},
    we can find an $\ep$-net $\mathcal{F} \subseteq \mathcal{K}_\lambda \subseteq \caH$ with respect to sup-norm of
    $\mathcal{K}_\lambda$ such that
    \begin{align}
      \label{eq:Proof_L2NormApprox_Covering}
      \abs{\mathcal{F}} \leq C \left( \lambda \ep \right)^{-\frac{2d}{s}},
    \end{align}
    where $\ep = \ep(n)$ will be determined later.

    Then, applying \cref{prop:SampleNormEstimation} to $\mathcal{F}$ with the $\norm{\cdot}_{L^\infty}$-bound \cref{eq:RegK_Inf},
    with probability at least $1-\delta$ we have
    \begin{align}
      \label{eq:L2nNormEstimation}
      \frac{1}{2}\norm{f}_{L^2}^2 - R_1
      \leq \norm{f}_{L^2,n}^2 \leq \frac{3}{2}\norm{f}_{L^2}^2 + R_1,
      \quad \forall f \in \mathcal{F},
    \end{align}
    where $ R_1 = C M_{\alpha}^4 \lambda^{-2\alpha} n^{-1} \ln(2\abs{\mathcal{F}} / \delta)$.

    Now, since $\mathcal{F}$ is an $\ep$-net of
    $\mathcal{K}_\lambda$ with respect to $\norm{\cdot}_{\infty}$,
    for any $x \in \mathcal{X}$, there exists some $f \in \mathcal{F}$ such that
    \begin{align*}
      \norm{T_{\lambda}^{-1} h_x - f}_{\infty} \leq \ep,
    \end{align*}
    which implies that
    \begin{align*}
      \abs{\norm{T_{\lambda}^{-1} h_x}_{L^2} - \norm{f}_{L^2}} \leq \ep,\quad
      \abs{\norm{T_{\lambda}^{-1} h_x}_{L^2,n} - \norm{f}_{L^2,n}} \leq \ep.
    \end{align*}
    Moreover, from \cref{eq:RegK_Inf} again we know that for $\mu$-a.e. $x \in \caX$,
    \begin{align*}
      \norm{T_{\lambda}^{-1} h_x}_{L^2},~ \norm{T_{\lambda}^{-1} h_x}_{L^2,n} \leq M_{\alpha} \lambda^{-\alpha},
    \end{align*}
    which is also true for $f$ since $f \in \mathcal{K}_{\lambda}$.
    Using $a^{2}-b^{2}=(a-b)(a+b)$, we get
    \begin{equation}
      \label{eq:L2NormEstimation}
      \begin{aligned}
        \abs{\norm{T_{\lambda}^{-1} h_x}_{L^2}^2 - \norm{f}_{L^2}^2} \leq 2 M_\alpha \ep \lambda^{-\alpha}, \\
        \abs{\norm{T_{\lambda}^{-1} h_x}_{L^2,n}^2 - \norm{f}_{L^2,n}^2} \leq 2 M_\alpha \ep \lambda^{-\alpha}.
      \end{aligned}
    \end{equation}

    Without loss of generality we consider only the upper bound:
    \begin{align*}
      \norm{T_{\lambda}^{-1} h_x}_{L^2,n}^2 & \leq \norm{f}_{L^2,n}^2 + 2 M_\alpha \ep \lambda^{-\alpha} \qq{(by \cref{eq:L2NormEstimation})} \\
      \qq{(by \cref{eq:L2nNormEstimation})}& \leq \frac{3}{2}\norm{f}_{L^2}^2 + R_1 + 2 M_\alpha \ep \lambda^{-\alpha}\\
      \qq{(by \cref{eq:L2NormEstimation} again)} & \leq \frac{3}{2}\norm{T_{\lambda}^{-1} h_x}_{L^2}^2
      + R_1 + 4 M_\alpha \ep \lambda^{-\alpha}  \\
      & = \frac{3}{2}\norm{T_{\lambda}^{-1} h_x}_{L^2}^2 + R_2,
    \end{align*}
    where $R_2 = C M_{\alpha}^4 n^{-1} \lambda^{-2\alpha}\ln(2\abs{\mathcal{F}}/\delta) + 4 M_\alpha \ep \lambda^{-\alpha}$.
    Letting $\ep = n^{-1}$ and applying in \cref{eq:Proof_L2NormApprox_Covering},
    the second term in $R_2$ is infinitesimal to the first one, and we have
    \begin{align*}
      R_2 \leq C M_{\alpha}^4 \frac{\lambda^{-2\alpha}}{n}\left( \ln \lambda + \ln n + \ln \frac{1}{\delta} \right)
      \leq C M_{\alpha}^4\frac{\lambda^{-2\alpha}}{n} \ln \frac{n}{\delta},
    \end{align*}
    since $\lambda = \Omega\left(n^{-p}\right)$.
  \end{proof}

  \subsection{Approximation A}

  The following proposition is a slightly modified version of \citet[Lemma 17]{fischer2020_SobolevNorm}.
  Compared with \citet[Lemma C.2]{li2023_SaturationEffect},
  it weakens the requirement of $\lambda$ using the embedding property.
  \begin{proposition}
    \label{prop:ConcenIneq}
    Suppose $\caH$ has embedding index $\alpha_0$ and \cref{assu:EDR} holds.
    Let $\lambda = \lambda(n) \to 0$ satisfy
    $\lambda = \Omega\left(n^{-1/(\alpha_0 + p)}\right)$ for some $p > 0$
    and fix arbitrary $\alpha \in (\alpha_0,\alpha_0 + p)$.
    Then, for all $\delta \in (0,1)$,
    when $n$ is sufficiently large (depending on $\delta,p,\alpha$), with probability at least $1 - \delta$,
    \begin{equation}
      \norm{T_{\lambda}^{-\frac{1}{2}} (T - T_X) T_{\lambda}^{-\frac{1}{2}} }_{\mathscr{B}(\caH)}
      \leq C M_\alpha \left( \frac{\lambda^{-\alpha}}{n}  \ln \frac{n}{\delta} \right)^{1/2},
    \end{equation}
    where $C > 0$ is a constant no depending on $n,\delta, p,\alpha$,
    and we also have
    \begin{align}
      \norm{T_{X\lambda}^{-1/2} T_{\lambda}^{1/2}}_{\mathscr{B}(\caH)} \leq \sqrt{3}.
    \end{align}
  \end{proposition}
  \begin{proof}
    We apply \cref{prop:ConcenIneq1}.
    We use $\norm{\cdot}$ for the operator norm $\norm{\cdot}_{\mathscr{B}(\caH)}$ in the proof for simplicity.
    By \cref{prop:EffectiveDimEstimation}, $\mathcal{N}(\lambda) \leq C \lambda^{-1/\beta}$.
    Moreover, since $\lambda = \Omega\left(n^{-1/(\alpha_0 + p)}\right)$ for some $p > 0$
    and $\alpha \in (\alpha_0,\alpha_0 + p)$, we have
    \begin{align*}
      u &= \frac{M_{\alpha}^{2} \lambda^{-\alpha}}{n} \ln{\frac{4 \mathcal{N}(\lambda) (\norm{T} + \lambda) }{\delta \norm{T}}} \\
      & \leq CM_{\alpha}^2 \frac{\lambda^{-\alpha}}{n} \left(  \ln \lambda^{-1/\beta} + \ln \frac{1}{\delta} + C \right) \\
      & \leq CM_{\alpha}^2 \frac{\lambda^{-\alpha}}{n} \left\{ (\alpha_0 + p)^{-1}\beta^{-1} \ln n + \ln \frac{1}{\delta} + C \right\} \\
      & \leq CM_{\alpha}^2 \frac{\lambda^{-\alpha}}{n} (\ln n + \ln \frac{1}{\delta}) = o(1),
    \end{align*}
    and thus
    \begin{align*}
      \norm{T_{\lambda}^{-\frac{1}{2}} (T - T_X) T_{\lambda}^{-\frac{1}{2}} } \leq
      \frac{4}{3} u + (2u)^{\frac{1}{2}} \leq C u^{\frac{1}{2}} \leq C M_\alpha \left( \frac{\lambda^{-\alpha}}{n}  \ln \frac{n}{\delta} \right)^{1/2}.
    \end{align*}

    For the second part, when $n$ is sufficiently large that $u \leq \frac{1}{8}$,
    \begin{align*}
      \norm{T_{\lambda}^{-\frac{1}{2}} (T - T_X) T_{\lambda}^{-\frac{1}{2}} }
      = \frac{4}{3} u + (2u)^{\frac{1}{2}} \leq \frac{2}{3}.
    \end{align*}
    Noticing that $\left( T_{X\lambda}^{-1/2} T_{\lambda}^{1/2} \right)^* = T_{\lambda}^{1/2}T_{X\lambda}^{-1/2}$,
    we have
    \begin{align*}
      \norm{T_{X\lambda}^{-1/2} T_{\lambda}^{1/2}}^2
      &= \norm{T_\lambda^{1/2} (T_X +\lambda)^{-1}T_\lambda^{1/2}} \\
      &= \norm{\left\{ T_\lambda^{-1/2} (T_X +\lambda)T_\lambda^{-1/2} \right\}^{-1}} \\
      &= \norm{\left\{ I - T_\lambda^{-1/2} (T - T_X)T_\lambda^{-1/2} \right\}^{-1}} \\
      & \leq \left(1-\norm{T_{\lambda}^{-\frac{1}{2}} (T - T_X) T_{\lambda}^{-\frac{1}{2}} }\right)^{-1} \leq 3.
    \end{align*}
  \end{proof}

  Then, the following lemma improves \citet[Lemma C.12]{li2023_SaturationEffect}
  and allows $\lambda$ to be of the order $n^{-\beta+\ep}$ if \cref{assu:EMB} is also satisfied.

  \begin{lemma}
    \label{lem:ApproxA}
    Suppose $\caH$ has embedding index $\alpha_0$ and \cref{assu:EDR} holds.
    Let $\lambda = \lambda(n) \to 0$ satisfy
    $\lambda = \Omega\left(n^{-1/(\alpha_0 + p)}\right)$ for some $p > 0$
    and fix arbitrary $\alpha \in (\alpha_0,\alpha_0 + p)$.
    Then, for sufficiently large $n$ (depending on $\delta,p,\alpha$), with probability at least $1-\delta$
    it is satisfied that for $\mu$-a.e. $x \in \caX$,
    \begin{align*}
      \abs{\norm{T_X^{1/2}T_{X\lambda}^{-1} h_x}_{\caH} - \norm{T_X^{1/2}T_{\lambda}^{-1} h_x}_{\caH}}
      \leq CM_{\alpha}^2 \left( \frac{\lambda^{-2\alpha}}{n}  \ln \frac{n}{\delta} \right)^{1/2},
    \end{align*}
    where the constant $C > 0$ do not depend on $n,\delta, p,\alpha$.
  \end{lemma}
  \begin{proof}
    We begin with
    \begin{align*}
      \abs{\norm{T_X^{1/2}T_{X\lambda}^{-1} h_x}_{\caH} - \norm{T_X^{1/2}T_{\lambda}^{-1} h_x}_{\caH}}
      & \leq \norm{T_X^{1/2}\left( T_{X\lambda}^{-1} - T_{\lambda}^{-1}\right) h_x}_{\caH} \\
      & = \norm{T_X^{1/2}T_{X\lambda}^{-1} \left( T - T_X \right) T_{\lambda}^{-1}h_x}_{\caH},
    \end{align*}
    where we notice that
    \begin{align*}
      T_{X\lambda}^{-1} - T_{\lambda}^{-1}
      = (T_X+\lambda)^{-1} -(T+\lambda)^{-1}
      = T_{X\lambda}^{-1} \left( T - T_X \right) T_{\lambda}^{-1}.
    \end{align*}

    Now let us decompose
    \begin{align*}
      &\quad \norm{T_X^{1/2}T_{X\lambda}^{-1} \left( T - T_X \right) T_{\lambda}^{-1}h_x}_{\caH} \\
      &= \norm{T_X^{1/2}T_{X\lambda}^{-1/2} \cdot T_{X\lambda}^{-1/2} T_{\lambda}^{1/2}
        \cdot  T_{\lambda}^{-1/2} \left( T - T_X \right)T_{\lambda}^{-1/2} \cdot
        T_{\lambda}^{-1/2}h_x}_{\caH} \\
      & \leq \norm{T_X^{1/2}T_{X\lambda}^{-1/2}}_{\mathscr{B}(\caH)}\cdot
      \norm{T_{X\lambda}^{-1/2} T_{\lambda}^{1/2}}_{\mathscr{B}(\caH)}
      \cdot \norm{T_{\lambda}^{-1/2} \left( T - T_X \right)T_{\lambda}^{-1/2} }_{\mathscr{B}(\caH)}
      \cdot \norm{T_{\lambda}^{-1/2}h_x}_{\caH} \\
      & \leq 1 \cdot \sqrt {3} \cdot C M_\alpha \left( \frac{\lambda^{-\alpha}}{n}  \ln \frac{n}{\delta} \right)^{1/2} \cdot M_{\alpha} \lambda^{-\alpha/2} \\
      & = C M_{\alpha}^2 \left( \frac{\lambda^{-2\alpha}}{n}  \ln \frac{n}{\delta} \right)^{1/2},
    \end{align*}
    where the four terms in the last inequality are:
    (1) from operator calculus and \cref{prop:FilterKRRControl};
    (2) and (3) from \cref{prop:ConcenIneq};
    (4) from \cref{eq:RegK_H}.

  \end{proof}

  \subsection{Final proof}


  \begin{theorem}
    \label{thm:VarLowerBound}
    Under Assumptions~\ref{assu:EDR},\ref{assu:EMB},\ref{assu:Holder} and~\ref{assu:noise},
    suppose that $\lambda = \lambda(n) \to 0$ satisfy
    $\lambda = \Omega\left(n^{-\beta+p}\right)$ for some $p > 0$.
    Then, for any $\delta \in (0,1)$, when $n$ is sufficiently large (depending on $\delta,p$),
    the following holds with probability at least $1-\delta$:
    \begin{align}
      \sigma^2 c \frac{\lambda^{-1/\beta}}{n} \leq V(\lambda) \leq \sigma^2 C \frac{\lambda^{-1/\beta}}{n},
    \end{align}
    where $c,C>0$ are absolute constants no depending on $\delta, n,\lambda$.
  \end{theorem}
  \begin{proof}
    We only prove the lower bound and the upper bound is similar.
    First, we assert that
    \begin{align}
      \label{eq:ProofVar_Approx}
      \norm{T_{X\lambda}^{-1}h_x}_{L^2,n}^2 \geq \frac{1}{2} \norm{T_\lambda^{-1}h_x}_{L^2}^2 - o(\lambda^{-1/\beta})
    \end{align}
    holds with probability at least $1-\delta$ for large $n$.
    Then, we have
    \begin{align*}
      V(\lambda) =  \frac{\sigma^2}{n}\int_{\mathcal{X}} \norm{T_{X\lambda}^{-1} h_x}_{L^{2},n}^2 \dd \mu(x)
      \geq \frac{\sigma^2}{2n}\int_{\mathcal{X}} \norm{T_\lambda^{-1} h_x}_{L^2}^2 \dd \mu(x)
      - o(\frac{\lambda^{-1/\beta}}{n}).
    \end{align*}
    For the integral, applying Mercer's theorem,
    we get
    \begin{equation}
      \label{eq:OracleNormEstimation}
      \begin{aligned}
        \int_{\mathcal{X}} & \norm{T_\lambda^{-1} h_x}_{L^2}^2 \dd \mu(x)
        = \int_{\mathcal{X}}\sum_{i =1}^\infty \left( \frac{\lambda_i}{\lambda + \lambda_i} \right)^2 e_i(x)^2 \dd \mu(x) \\
        &= \sum_{i =1}^\infty \left( \frac{\lambda_i}{\lambda + \lambda_i} \right)^2 = \mathcal{N}_2(\lambda) \geq c\lambda^{-1/\beta},
      \end{aligned}
    \end{equation}
    where the last inequality comes from \cref{prop:EffectiveDimEstimation}.
    Therefore, we conclude that
    \begin{align*}
      V(\lambda) &\geq \frac{c \sigma^2}{2n} \lambda^{-1/\beta} - o(\frac{\lambda^{-1/\beta}}{n})
      \geq c_1 \frac{\lambda^{-1/\beta}}{n}.
    \end{align*}

    Now we prove \cref{eq:ProofVar_Approx}.
    Let us choose $\alpha > \alpha_0 = 1/\beta$ sufficiently close but fixed.
    The requirement of $\alpha$ will be seen in the following proof.
    Then, \cref{lem:ApproxB} and \cref{lem:ApproxA} yield that
    \begin{gather}
      \label{eq:ProofVar_L2NormBound}
      \frac{1}{2} \norm{T_{\lambda}^{-1} h_x}_{L^2}^2 - R
      \leq \norm{T_{\lambda}^{-1} h_x}_{L^2,n}^2 \leq \frac{3}{2} \norm{T_{\lambda}^{-1} h_x}_{L^2}^2 + R, \\
      \label{eq:ProofVar_DiffControl}
      \abs{\norm{T_X^{1/2}T_{X\lambda}^{-1} h_x}_{\caH} - \norm{T_X^{1/2}T_{\lambda}^{-1} h_x}_{\caH}}\leq R^{1/2},
      \quad \forall x \in \caX,
    \end{gather}
    with probability at least $1-\delta$ for large $n$, where
    \begin{align*}
      R = C M_{\alpha}^4 \frac{\lambda^{-2\alpha}}{n} \ln \frac{n}{\delta}
    \end{align*}
    for some constant $C$ not depending on $n,\delta, p,\alpha$.
    Consequently, the $L^2$-norm bound in \cref{eq:RegK_L2}
    and \cref{eq:ProofVar_L2NormBound} yield
    \begin{align*}
      \norm{T_X^{1/2}T_{\lambda}^{-1} h_x}_{\mathcal{H}} & =  \norm{T_{\lambda}^{-1} h_x}_{L^2,n}
      \leq  C \norm{T_{\lambda}^{-1} h_x}_{L^2} + R^{1/2}
      \leq C \lambda^{-\alpha/2} + R^{1/2}
    \end{align*}
    which, together with \cref{eq:ProofVar_DiffControl}, implies
    \begin{align*}
      \norm{T_X^{1/2}T_{X\lambda}^{-1} h_x}_{\caH} + \norm{T_X^{1/2}T_{\lambda}^{-1} h_x}_{\caH}  \leq
      2\norm{T_X^{1/2}T_{\lambda}^{-1} h_x}_{\caH} + R^{1/2}
      \leq C (\lambda^{-\alpha/2} + R^{1/2})
    \end{align*}
    Then, we obtain the approximation of the squared norm
    \begin{align}
      \label{eq:ProofVar_ApproxTXT}
      \abs{\norm{T_X^{1/2}T_{X\lambda}^{-1} h_x}_{\caH}^2 - \norm{T_X^{1/2}T_{\lambda}^{-1} h_x}_{\caH}^2}
      \leq C R^{1/2} (\lambda^{-\alpha/2} + R^{1/2}).
    \end{align}
    Combining \cref{eq:ProofVar_ApproxTXT} and \cref{eq:ProofVar_L2NormBound}, we get
    \begin{align}
      \notag
      \norm{T_{X\lambda}^{-1} h_x}_{L^2,n}^2 & = \norm{T_X^{1/2}T_{X\lambda}^{-1} h_x}_{\caH}^2 \\
      \notag
      & \geq \norm{T_X^{1/2}T_{\lambda}^{-1} h_x}_{\caH}^2
      - C R^{1/2} (\lambda^{-\alpha/2} + R^{1/2}) \\
      \notag
      & = \norm{T_{\lambda}^{-1} h_x}_{L^2,n}^2 - C R^{1/2} (\lambda^{-\alpha/2} + R^{1/2}) \\
      \notag
      & \geq \frac{1}{2} \norm{T_{\lambda}^{-1} h_x}_{L^2}^2 - R - C R^{1/2} (\lambda^{-\alpha/2} + R^{1/2}) \\
      \label{eq:Pf_ErrorTerm}
      & = \frac{1}{2} \norm{T_{\lambda}^{-1} h_x}_{L^2}^2 - CR^{1/2}  (R^{1/2} + \lambda^{-\alpha/2}).
    \end{align}
    Finally, we show that the error terms in \cref{eq:Pf_ErrorTerm} are infinitesimal with respect to the main term $\lambda^{-1/\beta}$.
    We recall that $\lambda = \Omega(n^{-\beta+p})$ and $\alpha_0 = 1/\beta$, so
    \begin{align*}
      R^{1/2} =  C M_\alpha^2 \left(  \frac{\lambda^{-\alpha}}{n} \right)^{1/2} \lambda^{-\alpha/2} \left(\ln \frac{n}{\delta}\right)^{1/2}
      \leq C M_\alpha^2 \lambda^{-\alpha/2}
    \end{align*}
    if $1/\beta  < \alpha <  1/(\beta-p)$.
    Therefore,
    \begin{align*}
      \lambda^{1/\beta} R^{1/2}  (R^{1/2} + \lambda^{-\alpha/2}) & \leq
      C \lambda^{1/\beta} R^{1/2} \cdot  M_\alpha^2\lambda^{-\alpha/2} \\
      & \leq C M_\alpha^4 n^{-1/2} \lambda^{1/\beta-3\alpha/2} \left(\ln \frac{n}{\delta}\right)^{1/2},
    \end{align*}
    and thus $R^{1/2}  (R^{1/2} + \lambda^{-\alpha/2}) = o(\lambda^{-1/\beta})$ if
    we further set $\beta^{-1}  < \alpha < \beta^{-1}  \cdot (3\beta-2p)/(3\beta-3p)$.
  \end{proof}

  \begin{proof}
    [Proof of \cref{thm:NoGeneralization}]
    Let us choose $\lambda = n^{-\beta + p}$ for some $p > 0$.
    Then, from \cref{prop:VarianceTerm} and \cref{prop:VarOrder}, we have
    \begin{align*}
      E \left( \norm{\hat{f}_{\mathrm{inter}} - f^*_\rho}^2_{L^2} \mid X \right)
      \geq V(0) \geq V(\lambda).
    \end{align*}
    Moreover, \cref{thm:VarLowerBound} yields
    \begin{align*}
      V(\lambda) \geq c \frac{\lambda^{-1/\beta}}{n} = c \sigma^2 n^{-p/\beta}
    \end{align*}
    with high probability when $n$ is large.
    Since $p > 0$ can be arbitrary, we finish the proof.
  \end{proof}

  \section{Auxiliary results}

  \subsection*{Inequalities}

  \begin{proposition}
    \label{prop:FilterKRRControl}
    For $\lambda > 0$ and $s \in [0,1]$, we have
    \begin{align*}
      \sup_{t \geq 0} \frac{t^s}{t + \lambda} \leq \lambda^{s-1}.
    \end{align*}
  \end{proposition}

  \begin{proposition}
    \label{prop:EffectiveDimEstimation}
    Under \cref{assu:EDR}, for any $p \geq 1$, we have
    \begin{align}
      \label{eq:EffectDim}
      \mathcal{N}_p(\lambda)= \tr \left( TT_\lambda^{-1} \right)^p =
      \sum_{i = 1}^\infty \left( \frac{\lambda_i}{\lambda + \lambda_i} \right)^p \asymp \lambda^{-1/\beta}.
    \end{align}
  \end{proposition}
  \begin{proof}
    Since $c ~i^{-\beta} \leq \lambda_i \leq C i^{-\beta}$, we have
    \begin{align*}
      \mathcal{N}_{p}(\lambda) &= \sum_{i = 1}^{\infty} \left( \frac{\lambda_i}{\lambda_i + \lambda}  \right)^p
      \leq \sum_{i = 1}^{\infty} \left( \frac{C i^{-\beta}}{C i^{-\beta} + \lambda} \right)^p = \sum_{i = 1}^{\infty} \left( \frac{C }{C+ \lambda i^{\beta}}  \right)^p \\
      &\leq \int_{0}^{\infty} \left( \frac{C}{\lambda x^{\beta} + C} \right)^p\dd x
      = \lambda^{-1/\beta} \int_{0}^{\infty}  \left(\frac{C }{y^{\beta} + C} \right)^p\dd y \leq \tilde{C} \lambda^{-1/\beta}.
    \end{align*}
    for some constant $C$.
    The lower bound is similar.
  \end{proof}

  The following concentration inequality is introduced in \citet{caponnetto2010_CrossvalidationBased}.

  \begin{lemma}
    \label{lem:ConcenIneqVar}
    Let $\xi_1,\dots,\xi_n$ be $n$ i.i.d.\ bounded random variables such that $\abs{\xi_i} \leq B$ almost surely,
    $E(\xi_i) = \mu$, and $E (\xi_i - \mu)^2 \leq \sigma^2$.
    Then for any $\alpha > 0$, any $\delta \in (0,1)$, we have
    \begin{align}
      \label{eq:ConcenIneqVar}
      \abs{ \frac{1}{n}\sum_{i=1}^n \xi_i - \mu  } \leq \alpha \sigma^2 + \frac{3+4\alpha B}{6\alpha n} \ln \frac{2}{\delta}
    \end{align}
    holds with probability at least $1-\delta$.
  \end{lemma}

  Basing on \cref{lem:ConcenIneqVar}, we can obtain the following concentration about
  $L^2$ norm~\citep[Proposition C.9]{li2023_SaturationEffect}.

  \begin{proposition}
    \label{prop:SampleNormEstimation}
    Let $\mu$ be a probability measure on $\mathcal{X}$, $f \in L^2(\mathcal{X},\dd \mu)$ and $\norm{f}_{L^\infty} \leq M$.
    Suppose we have $x_1,\dots,x_n$ sampled i.i.d.\ from $\mu$.
    Then, the following holds with probability at least $1-\delta$:
    \begin{align}
      \frac{1}{2}\norm{f}_{L^2}^2 - \frac{5M^2}{3n}\ln \frac{2}{\delta} \leq \norm{f}_{L^2,n}^2 \leq
      \frac{3}{2}\norm{f}_{L^2}^2 + \frac{5M^2}{3n}\ln \frac{2}{\delta}.
    \end{align}
  \end{proposition}
  \begin{proof}
    Defining $\xi_i = f(x_i)^2$, we have
    \begin{align*}
      E (\xi_i) &= \norm{f}_{L^2}^2,\\
      E (\xi_i^2) & = E_{x \sim \mu} \left(f(x)^4\right) \leq \norm{f}_{L^\infty}^2 \norm{f}_{L^2}^2.
    \end{align*}
    Therefore, applying \cref{lem:ConcenIneqVar}, we get
    \begin{align*}
      \abs{\norm{f}_{L^2,n}^2 - \norm{f}_{L^2}^2} \leq
      \alpha \norm{f}_{L^\infty}^2 \norm{f}_{L^2}^2 + \frac{3+4\alpha M^{2}}{6\alpha n} \ln \frac{2}{\delta}.
    \end{align*}
    The final result comes from choosing $\alpha = (2M^2)^{-1}$.
  \end{proof}

  The following Bernstein type concentration inequality about self-adjoint Hilbert-Schmidt operator valued random variable
  is commonly used in the related literature~\citep{fischer2020_SobolevNorm,li2023_SaturationEffect}.

  \begin{lemma}
    \label{lem:ConcenBernstein}
    Let $H$ be a separable Hilbert space.
    Let $A_1,\dots,A_n$ be i.i.d.\ random variables taking values of self-adjoint Hilbert-Schmidt operators
    such that $E (A_1) = 0$, $\norm{A_1} \leq L$ almost surely for some $L > 0$ and
    $E (A_1^2) \preceq V$ for some positive trace-class operator $V$.
    Then, for any $\delta \in (0,1)$, with probability at least $1-\delta$ we have
    \begin{align*}
      \norm{\frac{1}{n}\sum_{i=1}^n A_i} \leq \frac{2LB}{3n} + \left(\frac{2\norm{V}B}{n}\right)^{1/2},
      \quad B = \ln \frac{4 \tr V}{\delta \norm{V}}.
    \end{align*}
  \end{lemma}

  A application of \cref{lem:ConcenBernstein} is the following~\citep[Lemma 17]{fischer2020_SobolevNorm}
  and we provide the proof for completeness.
  For simplicity, we use $\norm{\cdot}$ for the operator norm $\norm{\cdot}_{\mathscr{B}(\caH)}$.

  \begin{proposition}
    \label{prop:ConcenIneq1}
    Suppose $\caH$ has embedding index $\alpha_0$ and let $\alpha > \alpha_0$.
    Then, for all $\delta \in (0,1)$,
    when $n$ is sufficiently large, with probability at least $1 - \delta$,
    \begin{equation}
      \norm{T_{\lambda}^{-\frac{1}{2}} (T - T_X) T_{\lambda}^{-\frac{1}{2}} }
      \leq \frac{4}{3} u + (2u)^{\frac{1}{2}},
    \end{equation}
    where
    \begin{equation}
      \label{eq:LambdaCond}
      u = \frac{M_{\alpha}^{2} \lambda^{-\alpha}}{n} \ln{\frac{4 \mathcal{N}(\lambda) (\norm{T} + \lambda) }{\delta \norm{T}}},
    \end{equation}
    and $\mathcal{N}(\lambda) = \mathcal{N}_1(\lambda) = \tr (TT_\lambda^{-1})$.
  \end{proposition}
  \begin{proof}
    We will prove by \cref{lem:ConcenBernstein}.
    Let
    \begin{align*}
      A(x) = T_\lambda^{-\frac{1}{2}}(T_x - T)T_\lambda^{-\frac{1}{2}}
    \end{align*}
    and $A_i = A(x_i)$.
    Then, $E (A_i) = 0$ and
    \begin{align*}
      \frac{1}{n}\sum_{i=1}^n A_i = T_\lambda^{-\frac{1}{2}}(T_x - T)T_\lambda^{-\frac{1}{2}}.
    \end{align*}
    Moreover, since
    \begin{align*}
      T_\lambda^{-\frac{1}{2}} T_x T_\lambda^{-\frac{1}{2}}
      = T_\lambda^{-\frac{1}{2}} K_x K_x^* T_\lambda^{-\frac{1}{2}} =
      T_\lambda^{-\frac{1}{2}} K_x \left[ T_\lambda^{-\frac{1}{2}} K_x \right]^*,
    \end{align*}
    from \cref{eq:RegK_H} we have
    \begin{align}
      \label{eq:Proof_OpNormBound}
      \norm{T_\lambda^{-\frac{1}{2}} T_x T_\lambda^{-\frac{1}{2}}} =
      \norm{T_\lambda^{-\frac{1}{2}} K_x}_{\mathscr{B}(\R,\caH)}^2
      = \norm{T_\lambda^{-\frac{1}{2}} h_x}_{\caH}^2 \leq M_{\alpha}^2 \lambda^{-\alpha}.
    \end{align}
    By taking expectation,
    we also have $\norm{T_\lambda^{-\frac{1}{2}} T T_\lambda^{-\frac{1}{2}}} \leq M_{\alpha}^2 \lambda^{-\alpha}$.
    Therefore, we get
    \begin{align*}
      \norm{A} \leq
      \norm{T_\lambda^{-\frac{1}{2}} T T_\lambda^{-\frac{1}{2}}} +
      \norm{T_\lambda^{-\frac{1}{2}} T_x T_\lambda^{-\frac{1}{2}}}
      \leq 2M_{\alpha}^2 \lambda^{-\alpha} = L.
    \end{align*}

    For the second part of the condition,
    using the fact that $E (B - E (B) )^2 \preceq E (B^2)$ for a self-adjoint operator $B$,
    where $\preceq$ denotes the partial order induced by positive operators,
    we have
    \begin{align*}
      E (A^2) \preceq E \left( T_\lambda^{-\frac{1}{2}} T_x T_\lambda^{-\frac{1}{2}} \right)^2
      \preceq M_{\alpha}^2 \lambda^{-\alpha} E \left( T_\lambda^{-\frac{1}{2}}T_x T_\lambda^{-\frac{1}{2}} \right)
      = M_{\alpha}^2 \lambda^{-\alpha} T T_\lambda^{-1} = V,
    \end{align*}
    where the second $\preceq$ comes from \cref{eq:Proof_OpNormBound}
    Finally,
    \begin{align*}
      \norm{V} &= M_{\alpha}^2 \lambda^{-\alpha} \norm{TT_\lambda^{-1}} =
      M_{\alpha}^2 \lambda^{-\alpha} \frac{\lambda_1}{\lambda+\lambda_1}
      = M_{\alpha}^2 \lambda^{-\alpha}  \frac{\norm{T}}{\norm{T}+\lambda} \leq M_{\alpha}^2 \lambda^{-\alpha}, \\
      \tr V & =  M_{\alpha}^2 \lambda^{-\alpha} \tr \left[ TT_\lambda^{-1} \right]
      = M_{\alpha}^2 \lambda^{-\alpha} \mathcal{N}(\lambda),\\
      B &= \ln \frac{4 \tr V}{\delta \norm{V}} = \ln \frac{4(\norm{T}+\lambda) \mathcal{N}(\lambda)}{\delta \norm{T}}.
    \end{align*}
  \end{proof}

  \subsection*{Covering number in reproducing kernel Hilbert space}

  Let us first recall the definition of covering numbers.

  \begin{definition}
    Let $(E,\norm{\cdot}_E)$ be a normed space and $A \subset E$ be a subset.
    For $\ep > 0$, we say $S \subseteq A$ is an $\ep$-net of $A$ if for any $a \in A$, there exists $s \in S$ such that
    $\norm{a-s}_E \leq \ep$.
    Moreover, we define the $\ep$-covering number of $A$ to be
    \begin{align}
      \mathcal{N}(A,\norm{\cdot}_{E},\ep)
      &= \inf\left\{ \abs{S} : S \text{ is an $\ep$-net of } A \right\}.
    \end{align}
  \end{definition}

  \citet[Lemma C.10]{li2023_SaturationEffect} shows the following upper bound of covering numbers of
  $\mathcal{K}_{\lambda} = \left\{ T_\lambda^{-1} h_x \right\}_{x \in \caX}$.

  \begin{lemma}
    \label{lem:CoveringRegularK}
    Assuming that $\caX \subseteq \R^d$ is bounded and $k \in C^{0,s}(\caX \times \caX)$ for some $s \in (0,1]$.
    Then, we have
    \begin{align}
      \label{eq:CoveringSupRegularK}
      \mathcal{N}\left( \mathcal{K}_{\lambda}, \norm{\cdot}_{\infty}, \ep\right)
      & \leq C \left( \lambda \ep \right)^{-\frac{2d}{s}},
    \end{align}
    where $C$ is a positive constant not depending on $\lambda$ or $\ep$.
  \end{lemma}

\end{document}